\documentclass[letterpaper]{article} 
\usepackage{aaai2026}  
\usepackage{times}  
\usepackage{helvet}  
\usepackage{courier}  
\usepackage[hyphens]{url}  
\usepackage{graphicx} 
\usepackage{natbib}  
\usepackage{caption} 
\usepackage{algorithm}
\usepackage[noend]{algorithmic}

\usepackage{enumitem}
\usepackage{booktabs}
\usepackage{amsmath}
\usepackage{amssymb}
\usepackage[english]{babel}
\usepackage{amsthm}
\usepackage{tabulary}
\usepackage{multirow}
\usepackage{xcolor}

\newcommand{\appref}[1]{Appendix~\ref{#1}}
\newcommand{\sectref}[1]{Section~\ref{#1}}

\newcommand{\figref}[1]{Figure~\ref{#1}}
\newcommand{\tabref}[1]{Table~\ref{#1}}

\newcommand{\thmref}[1]{Theorem~\ref{#1}}

\newcommand{\agref}[1]{Algorithm~\ref{#1}}

\newcommand{\sectsectref}[2]{Sections~\ref{#1} and \ref{#2}}

\newcommand{\tabtabref}[2]{Tables~\ref{#1} and \ref{#2}}

\newcommand{\figfigfigref}[3]{Figures~\ref{#1}, \ref{#2} and \ref{#3}}
\newcommand{\tabtabtabref}[3]{Tables~\ref{#1}, \ref{#2} and \ref{#3}}

\newtheorem{theorem}{Theorem}

\newcommand{\startpara}[1]{{\vskip1pt\noindent{\bf #1.}}} 
\renewcommand{\url}[1]{{\def~{\char126}\sf#1}}

\def\cD{{\mathcal{D}}}
\def\cE{{\mathcal{E}}}
\def\cF{{\mathcal{F}}}

\def\cX{{\mathcal{X}}}

\def\cT{{\mathcal{T}}}

\def\cV{{\mathcal{V}}}

\def\cX{{\mathcal{X}}}

\def\sA{{\mathsf{A}}}
\def\sB{{\mathsf{B}}}
\def\sC{{\mathsf{C}}}
\def\sD{{\mathsf{D}}}
\def\sE{{\mathsf{E}}}

\newcommand{\trace}[1]{\mathsf{trace}{(#1)}}

\title{Explaining Decentralized Multi-Agent Reinforcement Learning Policies}
\author {
    Kayla Boggess\textsuperscript{\rm 1},
    Sarit Kraus\textsuperscript{\rm 2},
    Lu Feng\textsuperscript{\rm 1}
}
\affiliations {
    \textsuperscript{\rm 1}University of Virginia\\
    \textsuperscript{\rm 2}Bar-Ilan University\\
    \{kjb5we, lu.feng\}@virginia.edu,
    sarit@cs.biu.ac.il
}

\begin{document}

\maketitle

\begin{abstract}
Multi-Agent Reinforcement Learning (MARL) has gained significant interest in recent years, enabling sequential decision-making across multiple agents in various domains. However, most existing explanation methods focus on centralized MARL, failing to address the uncertainty and nondeterminism inherent in decentralized settings. We propose methods to generate policy summarizations that capture task ordering and agent cooperation in decentralized MARL policies, along with query-based explanations for “When,” “Why Not,” and “What” types of user queries about specific agent behaviors. We evaluate our approach across four MARL domains and two decentralized MARL algorithms, demonstrating its generalizability and computational efficiency. User studies show that our summarizations and explanations significantly improve user question-answering performance and enhance subjective ratings on metrics such as understanding and satisfaction.
\end{abstract}



\section{Introduction} \label{sec:intro} 

Multi-Agent Reinforcement Learning (MARL) has gained significant interest in recent years, enabling multi-agent sequential decision-making across various domains such as autonomous driving~\cite{dinneweth2022multi} and multi-robot warehousing~\cite{krnjaic2022scalable}. 
Recent works have explored generating explanations for MARL policies to enhance system transparency, improve user understanding, and foster human-agent collaboration~\cite{boggess2022toward,boggess2023explainable}. 
However, these prior efforts are primarily limited to \emph{centralized} MARL frameworks, where joint policies are learned and executed with full observability.
Such methods cannot adequately address the uncertainty, nondeterminism, and limited observability inherent in \emph{decentralized} MARL settings, which are common in real-world applications with communication or scalability constraints.

This work addresses this gap by introducing methods for generating policy summarizations and query-based explanations for decentralized MARL policies. 
Our approach is the first to summarize and explain agent coordination and task ordering under decentralized execution, enabling users to interpret multi-agent behavior even when individual agents act independently and only have local observations.

For example, consider a search and rescue mission where multiple cooperative robots follow a decentralized MARL policy.
A human operator in the field receives decision-making support via an explainer that provides high-level policy summaries and answers user queries using real-time trajectory data.
The summaries help the operator understand general robot behaviors—task completion, agent cooperation, task order—while the query-based explanations answer specific questions, such as:
\emph{``When do [agents] complete [task]?’’},
\emph{``Why don’t [agents] complete [task] under [conditions]?’’}, or
\emph{``What do the agents do after [task]?’’}.
With this information, the operator can make informed decisions, such as prioritizing urgent tasks or allocating resources more effectively.

A key challenge in supporting such explanations is representing the uncertain, asynchronous execution of decentralized policies.
Each agent's policy governs only its local behavior, possibly unaware of others' actions, making it difficult to infer global task order or cooperation from raw trajectories alone.

To tackle this, we develop a novel algorithm that constructs \emph{Hasse diagram}-based summarizations from trajectories generated under decentralized execution.
Each diagram is a directed acyclic graph where nodes represent tasks (annotated with the agents that completed them), and edges encode partial-order constraints over task completion times.
The resulting diagrams compactly capture both coordination and uncertainty: branching edges represent nondeterminism in task order, while nodes annotated with multiple agents indicate cooperation on shared tasks.

Building on this, we develop query-based explanation methods for three types of user queries: “When?”, “Why not?”, and “What?”.
Given a set of Hasse diagrams, we derive abstract states that encode key features such as task completions and agent involvement.
To capture uncertainty, we introduce an \emph{uncertainty dictionary} derived from partial comparability graphs that summarize unordered task dependencies across episodes.
We then apply the Quine-McCluskey algorithm~\cite{quine1952problem} to extract minimal Boolean formulas, which are translated into natural language explanations using structured templates, with uncertain features explicitly expressed using ``may'' conditions.

We evaluate our method’s generalizability and computational efficiency across four benchmark MARL domains, scaling to settings with up to 19 tasks and 9 agents.
To demonstrate the algorithm-agnostic nature of our method, we apply it to two distinct MARL algorithms that both yield decentralized policies but differ in their training paradigms: centralized training versus decentralized training.

Finally, we assess the effectiveness of our summarizations and explanations via two user studies measuring objective task performance and subjective ratings.
Results show that our approach significantly improves user question-answering accuracy and boosts subjective ratings such as understanding and satisfaction. 

Together, these contributions bridge the gap between opaque decentralized MARL policies and interpretable, human-centered explanations, enabling effective human-agent collaboration in multi-agent environments.

\section{Related Work} \label{sec:related} 

\startpara{Multi-Agent Reinforcement Learning}
MARL algorithms are commonly categorized by their training and execution paradigms.
\emph{Centralized training and centralized execution} (CTCE) methods train a single policy using full observability of the environment and agent states~\cite{marl-book}. 
\emph{Centralized training with decentralized execution} (CTDE) methods, such as SEAC~\cite{christianos2020shared}, use global information during training but deploy policies that operate on local observations at execution time.
\emph{Decentralized training and decentralized execution} (DTDE) methods, including independent learning approaches~\cite{papoudakis2021benchmarking}, train each agent’s policy independently, treating other agents as part of the environment.
This work focuses on post-hoc summarization and explanation of both CTDE and DTDE policies in cooperative settings.

\startpara{Policy Summarization}
Explainable RL (XRL) has received increasing attention, as surveyed in~\cite{milani2023explainable,wells2021explainable}, though most prior work targets single-agent settings.
For instance, \cite{topin2019generation} introduces \emph{abstract policy graphs}, which represent agent behavior as Markov chains over abstract states, and \cite{mccalmon2022caps} improves their comprehensibility.
\cite{amir2018highlights} visualizes agent behavior via representative trajectory videos.
In the multi-agent domain, \cite{milani2022maviper} uses intrinsically interpretable decision trees, while \cite{boggess2022toward} abstracts centralized MARL policies into macro-actions over joint state-action trajectories to identify common agent behaviors.
However, all of these methods assume a single input policy, whether from a single agent or a centralized multi-agent controller.
In contrast, we consider decentralized settings where each agent has its own policy and actions may require inter-agent coordination. We develop a scalable method to summarize such decentralized executions using compact, structured representations.

\startpara{Query-Based Explanations}
Post-hoc explanations in single-agent RL often rely on abstract state representations~\cite{hayes2017improving,sreedharan2022bridging}, saliency maps~\cite{atrey2019exploratory}, causal models~\cite{madumal2020explainable}, or reward decompositions~\cite{juozapaitis2019explainable}.
Several works have extended these ideas to multi-agent systems, but typically assume centralized control or non-cooperative agents.
For example, \cite{boggess2022toward} applies abstract policy graphs to centralized MARL, \cite{heuillet2022collective,mahjoub2024efficiently,chen2025understanding} compute agent contributions to joint policies, and \cite{kottinger2021maps} visualizes action assignments in joint plans.
\cite{mualla2022quest} proposes a framework for generating parsimonious explanations for teams of BDI agents.
However, these methods do not support inter-agent cooperation under decentralized policies and often aggregate agent behavior in a naive or disjointed manner.

\emph{To our knowledge, this is the first work to generate both policy summarizations and query-based explanations for decentralized MARL policies.}

\section{Policy Summarization} \label{sec:sum} 

\startpara{Decentralized MARL Policies}
Consider $N$ agents, each with a decentralized MARL policy $\pi^i: s^i \rightarrow \Delta(a^i)$ mapping local state $s^i$ to a distribution over actions $a^i$. Agents act asynchronously due to decentralized execution without a global clock. Joint tasks are assumed to be completed simultaneously, with each agent observing only its own contribution and reward.
Executing these policies yields trajectories $\{\omega^i\}_{i=1}^N$, where $\omega^i = s^i_0, a^i_0, r^i_0, s^i_1, \cdots$ records transitions observed by agent $i$. A task sequence $\trace{\omega^i} = \tau^i_1, \tau^i_2, \cdots$ can be extracted from each trajectory $\omega^i$, where a completed task $\tau$ is inferred from reward signals and state transitions.

\startpara{Problem Statement}
Given decentralized MARL policies and their execution trajectories, how can we generate a compact, interpretable summary that captures both individual and joint agent behaviors? We seek a representation that ensures \emph{correctness}, meaning each agent’s behavior in the summary aligns with its actual task sequence, and \emph{completeness}, meaning each agent’s full task sequence is captured in at least one path in the summary.

\begin{figure*}[t]
    \centering
    \includegraphics[width=0.8\linewidth]{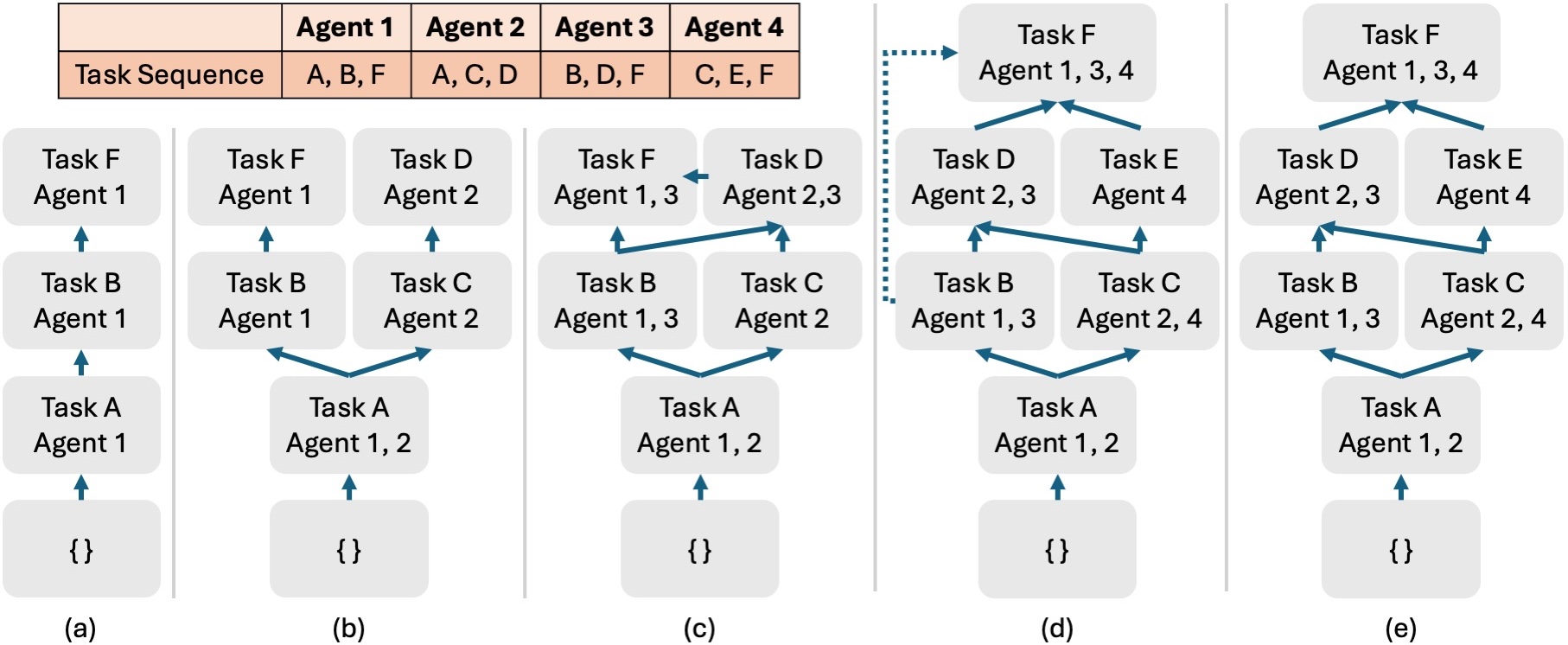}
    \caption{Example of \agref{ag:hds} constructing a Hasse diagram incrementally: steps (a)–(d) incorporate each agent's task sequence, and step (e) applies transitive reduction.}
    \label{fig:example}
\end{figure*}

\startpara{Hasse Diagram Summarization}
We propose to summarize decentralized agent behavior using a Hasse diagram $\cD = (\cV, \cE)$, a directed acyclic graph that represents a partial order over task completions~\cite{sarkar2017textbook}. Each vertex denotes a set of tasks completed simultaneously and the agents that perform them. Edges encode precedence: $v \prec v’$ indicates that tasks in $v$ must precede those in $v’$.
A path $\rho = v_0 \to v_1 \to \cdots$ through $\cD$ defines a possible task ordering. The projection of $\rho$ onto agent $i$, denoted $\rho^i$, retains only tasks performed by agent $i$. We say $\rho^i$ conforms to the agent’s trajectory $\trace{\omega^i}$ if $\rho^i \sqsubseteq \trace{\omega^i}$—i.e., the task order is preserved.

Formally, the diagram $\cD$ is \emph{correct} if, for all paths $\rho$ and agents $i$, either $\rho^i = \emptyset$ or $\rho^i \sqsubseteq \trace{\omega^i}$. It is \emph{complete} if, for every agent $i$, there exists a path $\rho$ such that $\rho^i = \trace{\omega^i}$.

We present \agref{ag:hds}, which constructs a correct and complete Hasse diagram $\cD$ from a single episode of decentralized execution trajectories $\{\omega^i\}_{i=1}^N$, agnostic to how the policies are trained (e.g., CTDE, DTDE). The algorithm iterates through task sequences, creates nodes for new tasks, identifies agents for shared tasks, and inserts edges to maintain local task order. Finally, a transitive reduction is applied to eliminate redundant edges. 

\figref{fig:example} shows an example of applying \agref{ag:hds}, illustrating the incremental construction of a Hasse diagram summarizing agents' behavior.

\begin{algorithm}[t]
\caption{Hasse Diagram Summarization (HDS)}
\label{ag:hds}
{\small
\textbf{Input}: Agent trajectories $\{\omega^i\}_{i=1}^N$ \\
\textbf{Output}: Hasse Diagram $\cD = (\cV, \cE)$
\begin{algorithmic}[1]
\STATE Initialize: $v_0 \gets \emptyset$; $\cV \gets \{v_0\}$; $\cE \gets \emptyset$
\FOR{each agent $i = 1$ to $N$}
    \STATE $T^i \gets \trace{\omega^i}$
    \FOR{each task index $k = 1$ to $|T^i|$}
        \STATE $\tau \gets T^i_k$
        \IF{$\tau$ exists in some vertex $v \in \cV$}
            \STATE Add agent $i$ to $v[\tau]$
        \ELSE
            \STATE Create new vertex $v'$ with $v'[\tau] = \{i\}$; $\cV \gets \cV \cup \{v'\}$
        \ENDIF
        \IF{$k = 1$}
            \STATE Add edge $(v_0 \to v)$ or $(v_0 \to v')$ to $\cE$
        \ELSE
            \STATE Let $\tau_{\text{prev}} \gets T^i_{k-1}$
            \STATE Find vertex $\bar{v} \in \cV$ containing $\tau_{\text{prev}}$
            \STATE Add edge $(\bar{v} \to v)$ or $(\bar{v} \to v')$ to $\cE$ if not present
        \ENDIF
    \ENDFOR
\ENDFOR
\FOR{each edge $(u \to v) \in \cE$}
    \IF{a path from $u$ to $v$ exists excluding edge $(u \to v)$}
        \STATE Remove edge $(u \to v)$ from $\cE$
    \ENDIF
\ENDFOR
\RETURN $\cD = (\cV, \cE)$
\end{algorithmic}
}
\end{algorithm}

\startpara{Complexity and Guarantee}
The worst-case time complexity of \agref{ag:hds} is $\mathcal{O}(N \cdot |T|^2 + |T|^4)$, where $N$ is the number of agents and $|T|$ is the number of tasks. The $\mathcal{O}(N \cdot |T|^2)$ term covers task-wise updates across $N$ trajectories, and the $\mathcal{O}(|T|^4)$ term arises from transitive reduction over a graph with up to $|T|^2$ edges.

\begin{theorem} \label{thm:hds}
Given a set of agent trajectories $\{\omega^i\}_{i=1}^N$ produced by executing decentralized MARL policies $\{\pi^i\}_{i=1}^N$ in a single episode, the Hasse diagram $\cD = (\cV, \cE)$ constructed by \agref{ag:hds} is both a correct and complete policy summarization. (Proof provided in \appref{app:proof}.)
\end{theorem}

\startpara{Practical Considerations}
In large environments, users are often interested in only a subset of agents or tasks (e.g., nearby robots in a search and rescue scenario). Our method supports selective summarization by restricting input to relevant agent trajectories and applying task filters during sequence extraction.

Because decentralized execution is stochastic, different episodes may yield different Hasse diagrams. To summarize observed behaviors, we rely on actual trajectories; to capture potential future behaviors, we can simulate multiple episodes and report a representative diagram, such as the most frequent one.

\section{Query-Based Explanations} \label{sec:query} 

While Hasse diagrams summarize global behavior, they do not explain local decisions—such as \emph{when} agents choose to perform a task, \emph{why} they fail to do so under certain conditions, or \emph{what} they do next. To address this gap, we develop methods that generate language-based explanations in response to user queries. Our approach builds on prior work in query-based explanations for single-agent~\cite{hayes2017improving} and centralized multi-agent settings~\cite{boggess2022toward}, but introduces new techniques to address the uncertainty and partial observability inherent in decentralized execution.

\subsection{Answering “When” Queries}\label{sec:when}

We consider queries of the form: \textit{“When do agents $\mathcal{G}_q$ perform task $\tau_q$?”}, aiming to identify the necessary conditions under which $\tau_q$ is completed by agent group $\mathcal{G}_q$ across multiple simulated executions. \agref{ag:when} outlines our method for generating language-based explanations for such queries.

We begin by extracting a subset of features $\mathcal{F}_q \subseteq \mathcal{F}$ that are relevant to the query task, using domain knowledge. For example, for the query “When do agents 2 and 4 do task $\sC$?", relevant features may include boolean predicates indicating whether agents 2 or 4 complete task $\sC$, as well as whether potentially prerequisite tasks (e.g., task $\sA$, $\sB$, etc.) have been completed.

\begin{algorithm}[t]
\caption{``When'' Query-Based Explanation}
\label{ag:when}
{\small
\textbf{Input:} Agent group $\mathcal{G}_q$, query task $\tau_q$, Hasse diagrams $\{\mathcal{D}_j\}$, and feature set $\mathcal{F}$ \\
\textbf{Output:} Language-based explanation $\cX$
\begin{algorithmic}[1]
\STATE Extract relevant features $\mathcal{F}_q \subseteq \mathcal{F}$ for $\tau_q$
\STATE Initialize uncertainty dictionary $U \leftarrow \emptyset$
\FOR{each diagram $\mathcal{D}_j$}
  \IF{$\tau_q$ completed by $\mathcal{G}_q$ in $\mathcal{D}_j$}
    \STATE Let $v_\tau$ be the node where $\tau_q$ is completed
    \STATE Compute partial comparability graph from $v_\tau$
    \FOR{each node $v$ not reachable to/from $v_\tau$}
      \STATE Add features associated with $v$ to $U[\mathcal{D}_j]$
    \ENDFOR
  \ENDIF
\ENDFOR
\STATE Label nodes as targets (containing $\tau_q$ by $\mathcal{G}_q$) or non-targets
\STATE Encode nodes as boolean vectors over $\mathcal{F}_q$ using $U$
\STATE Apply Quine-McCluskey to distinguish targets/non-targets
\STATE Translate resulting formula into explanation $\cX$ 
\RETURN $\cX$
\end{algorithmic}
}
\end{algorithm}

Given a set of Hasse diagrams $\{\mathcal{D}_j\}$ summarizing multiple episodes of decentralized policy execution, we check, for each diagram $\mathcal{D}_j$, whether the query task $\tau_q$ is completed by the queried agent group $\mathcal{G}_q$. If so, we identify the corresponding node $v_{\tau}$ and construct a \emph{partial comparability graph}~\cite{kelly1985comparability} centered at $v_{\tau}$, which includes only nodes with a known ordering relative to $v_{\tau}$ (i.e., those reachable via the Hasse diagram’s edges in either direction).

To handle partial observability and structural ambiguity in decentralized execution, we introduce an \emph{uncertainty dictionary} $U$. Any node that is not reachable to or from $v_{\tau}$ is considered unordered with respect to the query task. Features associated with such nodes are marked as \emph{uncertain} and stored in $U[\mathcal{D}_j]$.
For example, if task $\sB$ appears in a node unconnected to the node where agents 2 and 4 complete task $\sC$, we cannot determine whether task $\sB$ occurred before or after task $\sC$. As a result, the feature “task $\sB$ completed” is added to $U[\mathcal{D}_j]$ and treated as a possible—but not confirmed—precondition for task $\sC$.

We then label nodes as \emph{targets} if they satisfy the query (i.e., they contain $\tau_q$ completed by $\mathcal{G}_q$), and as \emph{non-targets} otherwise. Each node is encoded as a boolean vector over the relevant feature set $\mathcal{F}_q$, where each bit indicates whether the corresponding feature is satisfied along a path to that node. To avoid underestimating dependencies, features marked as uncertain in $U$ are conservatively treated as true in the boolean encoding.

To identify distinguishing conditions, we apply Quine-McCluskey~\cite{quine1952problem} to derive a minimal boolean formula that separates targets from non-targets. Finally, we translate the resulting formula into a language explanation using a structured language template, mapping certain features to “must” and uncertain ones to “may”. 

An example of a generated explanation is:  
“For agents 2 and 4 to complete task $\sC$, agent 2 must complete task $\sC$, agent 4 must complete task $\sC$, and task $\sA$ must be completed. Additionally, task $\sB$ \emph{may} need to be completed.”

While prior methods~\cite{hayes2017improving, boggess2022toward} also use Quine-McCluskey minimization followed by language translation, they do not account for the uncertainty introduced by decentralized execution. In contrast, \agref{ag:when} incorporates partial comparability graphs and an uncertainty dictionary to capture unordered task dependencies, enabling “may” conditions in the resulting explanations.

\startpara{Complexity}
The dominant cost of \agref{ag:when} is Quine-McCluskey minimization with worst-case complexity \(\mathcal{O}(3^{|\mathcal{F}_q|} / \ln |\mathcal{F}_q|)\). Other steps scale linearly with the number of Hasse diagrams and nodes. The method is tractable in practice for moderate feature sizes.

\subsection{Answering “Why Not” Queries} \label{sec:why}

To answer queries of the form: \emph{“Why don’t agents \(\mathcal{G}_q\) do task \(\tau_q\) under conditions \(\Phi_q\)?”}, we adapt the procedure used for “When” queries. Instead of identifying preconditions for successful completions, the goal is to isolate the minimal set of missing conditions that prevent \(\tau_q\) from occurring under the given scenario.

The key difference lies in how we define the target and non-target sets: the user-provided condition \(\Phi_q\) is encoded as the target (i.e., a case where the task did not occur), while nodes from Hasse diagrams where \(\tau_q\) is successfully completed by \(\mathcal{G}_q\) serve as non-targets. As in the “When” query, we construct partial comparability graphs to identify ordering uncertainty and maintain an uncertainty dictionary to track ambiguous dependencies. These are incorporated into the boolean encoding, allowing us to apply Quine-McCluskey minimization to identify which missing features distinguish the query condition from successful executions. The full algorithm pseudocode is provided in \appref{app:why} and shares the same complexity as \agref{ag:when}.

For instance, for the query “Why don’t agents 2 and 4 complete task $\sC$ when only task $\sA$ is completed?”, the resulting explanation could be: “Task $\sB$ \emph{may} need to be completed for agents 2 and 4 to complete task $\sC$.”

\subsection{Answering “What” Queries} \label{sec:what}

\begin{algorithm}[tb]
\caption{``What'' Query-Based Explanation}
\label{ag:what}
{\small
\textbf{Input:} Query task $\tau_q$, Hasse diagrams $\{\mathcal{D}_j\}$ \\
\textbf{Output:} Language-based explanation $\cX$
\begin{algorithmic}[1]
\STATE Initialize sets $\mathcal{T}^c \leftarrow \emptyset$, $\mathcal{T}^u \leftarrow \emptyset$
\FOR{each diagram $\mathcal{D}_j$}
  \FOR{each node $v \in \mathcal{D}_j$}
    \IF{$v$ contains $\tau_q$ completed by any agent group}
      \STATE Add tasks from all immediate children of $v$ to $\mathcal{T}^c$
      \STATE Compute partial comparability graph from $v$
      \FOR{each node $v'$ not reachable to/from $v$}
        \STATE Add tasks from $v'$ to $\mathcal{T}^u$
      \ENDFOR
    \ENDIF
  \ENDFOR
\ENDFOR
\STATE Translate $\mathcal{T}^c$ and $\mathcal{T}^u$ into explanation $\cX$
\RETURN $\cX$
\end{algorithmic}
}
\end{algorithm}

To answer queries of the form: \emph{“What do the agents do after task $\tau_q$?”}, we analyze the successors of $\tau_q$ across multiple Hasse diagrams. Our goal is to identify which tasks occur after $\tau_q$, distinguishing between those that are certainly ordered afterward and those that may follow, but whose ordering is ambiguous due to decentralized execution.

Given a set of Hasse diagrams $\{\mathcal{D}_j\}$ generated from simulated episodes, we first locate all nodes where $\tau_q$ is completed. For each such node, we add the tasks from its immediate children, representing actions that are explicitly ordered after $\tau_q$, to a set of certain successors $\mathcal{T}^c$.

To identify uncertain successors $\mathcal{T}^u$, we construct a partial comparability graph rooted at each node where $\tau_q$ is completed. We then collect tasks from nodes that are not ordered with respect to it. These tasks are added to a set $\mathcal{T}^u$ as possible, but not guaranteed, successors of $\tau_q$.

We generate an explanation using a language template that reports both the certain and uncertain successor sets. 
For example, for the query “What do agents do after task $\sC$ is completed?”, the explanation could be:  
“After task $\sC$ is completed, tasks $\sD$ and $\sE$ are completed. Additionally, task $\sB$ \emph{may} be completed.”

\startpara{Complexity}
The worst-case time complexity of \agref{ag:what} is \(\mathcal{O}(|\{\mathcal{D}_j\}| \cdot |\cV|^2(|\cV| + |\cE|))\), where \(|\{\mathcal{D}_j\}|\) is the number of Hasse diagrams, and \(|\cV|\) and \(|\cE|\) are the number of nodes and edges in each diagram, respectively.

\section{Computational Experiments} \label{sec:exp} 

\startpara{MARL Domains}
We evaluate our approaches on four benchmark domains:  
(1) \textit{Search and Rescue (SR)}~\cite{boggess2022toward}, where agents rescue victims and fight fires;  
(2) \textit{Level-Based Foraging (LBF)}~\cite{papoudakis2021benchmarking}, where agents collect food;  
(3) \textit{Multi-Robot Warehouse (RW)}~\cite{papoudakis2021benchmarking}, where agents pick up and deliver items; and  
(4) \textit{Pressure Plate (PP)}~\cite{pressureplate}, where agents open doors to enable others' navigation.  
All domains are gridworld-based. Agents observe nearby grid cells only: up to four per direction in PP and one per direction in other domains, reflecting the partial observability of decentralized execution.

\startpara{Experimental Setup}
We train policies using two MARL algorithms:
Shared Experience Actor-Critic (SEAC)~\cite{christianos2020shared} for CTDE, and
Independent Advantage Actor-Critic (IA2C)~\cite{papoudakis2021benchmarking} for DTDE.
Each model is trained until convergence or for up to 400 million steps.
All experiments are conducted on a machine with a 2.1 GHz Intel CPU, 132 GB RAM, and Ubuntu 22.04.

\subsection{Evaluation on Policy Summarization} \label{sec:exp-sum}

\startpara{Summarization Baseline}
Since no existing methods summarize decentralized MARL policies, we adapt the single-agent approach from~\cite{mccalmon2022caps} as a baseline. For each agent, we construct an \emph{abstract policy graph} that summarizes task sequences observed over 100 episodes, using the same abstract features as our HDS method for fair comparison. The resulting agent-specific graphs are displayed side-by-side and annotated with task sequence probabilities (see \figref{fig:baseline-sum} in \appref{app:exp}).

\startpara{Results}  
\tabref{tab:sum-ctde} compares CTDE policy summarization sizes (number of nodes and edges) produced by our HDS method and the baseline for the largest configuration in each domain (i.e., with the most agents and tasks).
HDS generates one Hasse diagram per episode, each representing a complete summary of all agents' behavior in that episode; we report average size over 100 episodes.
In contrast, the baseline displays all observed task sequences across episodes for all agents, resulting in large visualizations with hundreds of nodes and edges, which are significantly harder to interpret than the compact Hasse diagrams.
Similar results hold for DTDE policies trained with IA2C (see \appref{app:exp}).

Both HDS and the baseline are computationally efficient, each processing 100 episodes and generating summarizations in under one second across all domains.

Additionally, while HDS often produces unique Hasse diagrams across episodes, they typically fall into a small number of structural types based on edge counts. For example, SR(9,7) yields 100 unique diagrams but only 6 distinct edge counts. \figref{fig:piechart} in \appref{app:exp} illustrates the distribution of diagram types across domains, highlighting HDS’s ability to capture both behavioral diversity and structural regularity.

\begin{table}[t]
\centering
\begin{tabular}{l|cc|cc}
\toprule
\multicolumn{1}{c}{Domain} & \multicolumn{2}{c}{\textbf{HDS}} & \multicolumn{2}{c}{\textbf{Baseline}} \\
\cmidrule(lr){2-3} \cmidrule(lr){4-5}
\multicolumn{1}{c}{$(N, |T|)$} & $|\cV|$ & $|\cE|$ & $|\cV|$& $|\cE|$ \\
\midrule
SR (9,7)  & 8 & 7.88 & 534 & 525 \\
LBF (9,9) & 10 & 10.83 & 723 & 714 \\
RW (4,19) & 20 & 19 & 1,274 & 1,270 \\
PP (7,6)  & 7 & 6 & 265 & 258 \\
\bottomrule
\end{tabular}
\caption{Summarization sizes for HDS and the baseline on the largest setting in each domain, based on 100 episodes executed using CTDE policies trained with SEAC.}
\label{tab:sum-ctde}
\end{table}

\subsection{Evaluation on Query-Based Explanations} \label{sec:exp-query}

\startpara{Explanation Baseline}
Since no existing methods generate query-based explanations for decentralized MARL, we adapt the single-agent explanation approach from~\cite{hayes2017improving} by applying it independently to each agent's abstract policy graph (described in \sectref{sec:exp-sum}). The resulting per-agent explanations are then merged using simple union-based aggregation as a baseline.

\begin{table}[t]
\centering
\begin{tabular}{l|cc|cc}
\toprule
\multicolumn{1}{c}{Domain} & \multicolumn{2}{c}{\textbf{HD-When}} & \multicolumn{2}{c}{\textbf{Baseline}} \\
\cmidrule(lr){2-3} \cmidrule(lr){4-5}
\multicolumn{1}{c}{$(N, |T|)$} & $|\cF^c|$ & $|\cF^u|$ & $|\cF^c|$& $|\cF^u|$ \\
\midrule
SR (9,7)  & 9 & 2 & 54 & 0 \\
LBF (9,9) & 13 & 11 & 104 & 0 \\
RW (4,19) & 0 & 153 & 267 & 0 \\
PP (7,6)  & 8 & 3 & 20 & 0 \\
\bottomrule
\end{tabular}
\caption{Explanations sizes for our method and the baseline on the largest setting in each domain, based on 100 episodes executed using CTDE policies trained with SEAC.}
\label{tab:when-ctde}
\end{table}

\startpara{Results}
\tabref{tab:when-ctde} compares explanation sizes for “When” queries on CTDE policies, measured by the number of certain ($|\cF^c|$) and uncertain ($|\cF^u|$) features extracted by our method and the baseline. These features are derived from boolean formulas obtained through Quine-McCluskey minimization. In all cases, both methods generate explanations in under one second.

The baseline includes only certain features, as it does not capture task ordering uncertainty. It also produces significantly larger explanations due to the size of its aggregated policy graphs and union-based merging of per-agent results. In contrast, our HD-When method yields more compact and informative explanations, balancing certain and uncertain features. In RW(4,19), all features are marked uncertain due to highly asynchronous execution.

Similar trends are observed for DTDE policies and for other query types (see \appref{app:exp}).

\section{User Studies} \label{sec:study} 

We conducted two user studies (with IRB approval) to evaluate the effectiveness of our proposed policy summarizations and query-based explanations.

For both studies, we recruited participants via university mailing lists to answer surveys via Qualtrics. Eligible participants were fluent English speakers over 18 years old and were incentivized with bonus payments for correctly answering questions based on the provided summarizations or explanations. 
The summarization study included 20 participants (10 males, 9 females, 1 other), with an average age of 22.55 years (SD = 2.89). 
The explanation study included 21 participants (14 males, 6 females, 1 other), with an average age of 24 years (SD = 3.95).

We describe the study design and results for each study in \sectsectref{sec:study-sum}{sec:study-query}, respectively.

\subsection{Summarization Study} \label{sec:study-sum}

\startpara{Independent Variables}
The independent variable in this study was the summarization generation approach: our HDS method or the baseline described in \sectref{sec:exp-sum}. 
The baseline displays side-by-side abstract policy graphs for each agent, with each agent’s most likely task sequence highlighted (\figref{fig:baseline-sum}, \appref{app:exp}). 
To aid interpretation, the HDS interface presents a table of agent-task assignments and a list of task-ordering rules converted from each Hasse diagram. It shows the top three most frequent plans from 100 episodes, annotated with empirical likelihoods, and highlights the most likely one in red (\figref{fig:hds}, \appref{app:study}).

\startpara{Procedures}
The study followed a within-subject design where each participant completed two trials, one per summarization method (HDS and baseline). 
Each trial consisted of two summarizations, and each summarization was followed by three questions (12 questions total). 
Participants were randomly assigned to one of two groups to counterbalance ordering effects (HDS first or baseline first). 
All questions were randomized. Prior to each trial, participants received a brief tutorial and passed attention-check questions to ensure engagement. 
Bonus incentives and timing were used to promote data quality.

\startpara{Dependent Measures}
We assessed user performance based on the number of correctly answered questions in three categories: 
\emph{assignment} (e.g., “Can [robot(s)] complete [task]?”), 
\emph{likelihood} (e.g., “What are the most likely robot(s) to complete [task]?”), 
and \emph{order} (e.g., “Must [task 1] always be completed before [task 2]?”). 
Response time was also recorded for each question.

At the end of each trial, participants rated the summarization quality on a 5-point Likert scale across seven metrics~\cite{hoffman2018metrics}: 
\emph{understanding}, \emph{satisfaction}, \emph{detail}, \emph{completeness}, \emph{actionability}, \emph{reliability}, and \emph{trustworthiness}. 
Participants were informed of their accuracy before providing ratings.

\startpara{Hypotheses}
We hypothesized that, compared to baseline summarizations, HDS would (\textbf{H1}) improve user question-answering performance and (\textbf{H2}) receive higher ratings on summarization quality metrics.

\startpara{Results}
We found that users answered significantly more questions correctly using HDS (M=4.25 out of 6, SD=0.83) than with the baseline (M=3.1 out of 6, SD=1.04). A paired t-test confirms this difference is statistically significant (t(19)=4.2, p $\leq$ 0.01, d=0.96). \emph{The data supports H1.} 

Regarding user-perceived summarization quality, HDS was rated slightly higher in \emph{completeness} (Wilcoxon signed-rank, W=16.0, Z=-2.07, p $\leq$ 0.04, r=-0.33), but not significantly different in other dimensions (see \figref{fig:sumGoodRating}). \emph{The data partially supports H2.} 

\begin{figure}[t]
    \centering
    \includegraphics[width=1.0\columnwidth]{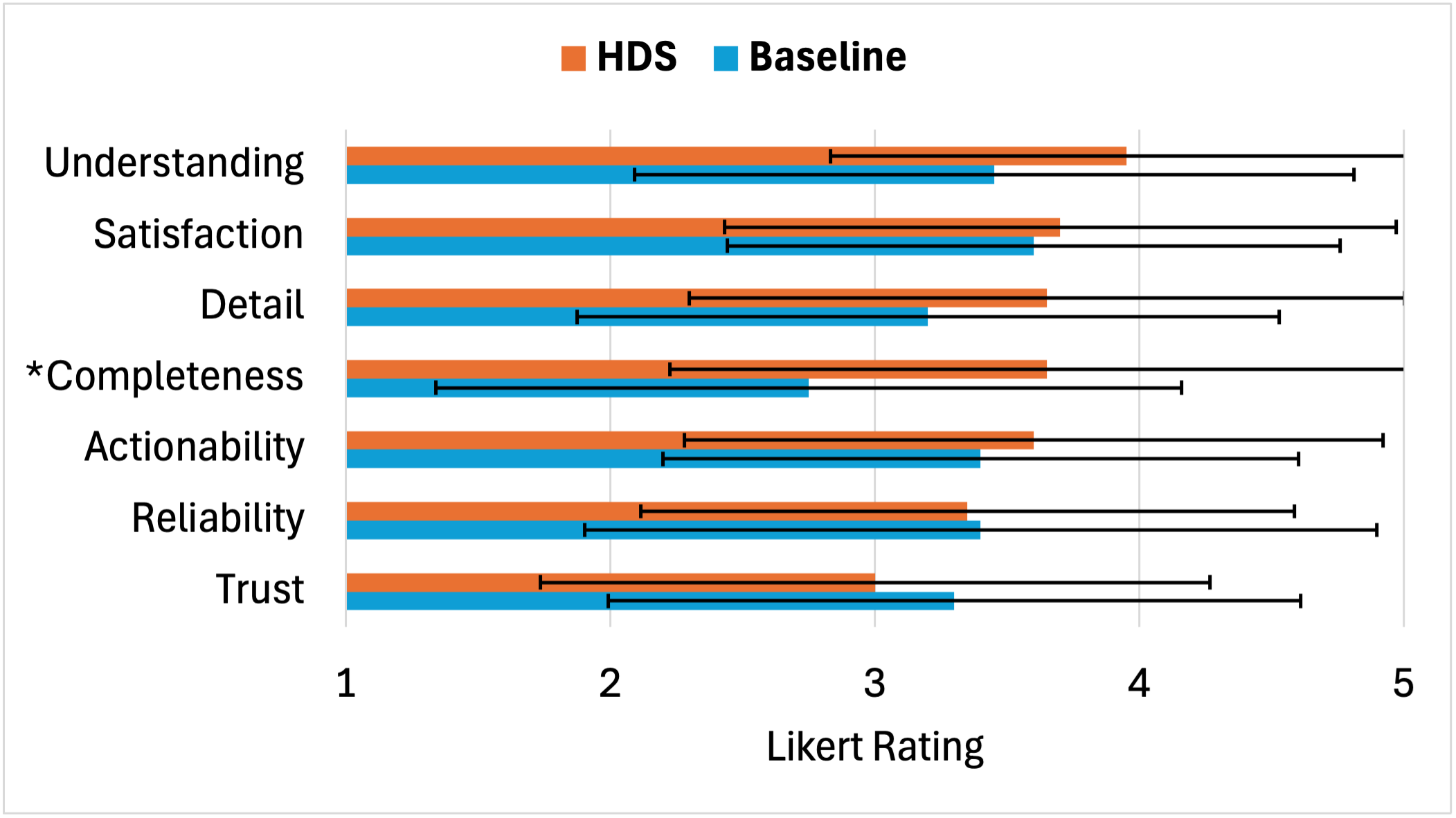}
    \caption{Mean and standard deviation of participant ratings on policy summarizations (* indicates statistically significant difference).}
    \label{fig:sumGoodRating}
\end{figure}

\startpara{Discussion}
These results suggest that HDS improves objective user performance in answering questions that require understanding task coordination across agents. 
Because the baseline does not explicitly model inter-agent cooperation, users must compare across graphs to infer coordination, which can be cognitively demanding. Moreover, the baseline may mislead by showing each agent’s most likely task sequence independently, which may not reflect the most likely joint behavior in coordination.

Subjective ratings show limited preference for HDS, possibly due to user familiarity with the baseline’s flowchart-style layout despite its reduced clarity on coordination.  

Finally, we note that users' response times remained comparable between methods. In some cases, HDS enabled faster responses (e.g., for likelihood queries).

\subsection{Explanation Study} \label{sec:study-query}

\startpara{Independent Variables}
The independent variable was the explanation generation method: our proposed approach versus the baseline described in \sectref{sec:exp-query}. 
Example interfaces for all three query types (“When,” “Why Not,” and “What”) are shown in \appref{app:study}. 

For “When” queries, users viewed a map of a search-and-rescue scenario (with four agents and four tasks) and an accompanying explanation. Our method outputs both required (certain) and possible (uncertain) conditions, while the baseline includes only certain ones. 

For “Why Not” queries, users were shown two maps: the first with a failed task and an explanation of the violated conditions, and the second used to test whether the explanation helped predict behavior. 

For “What” queries, users received an explanation about what tasks could occur next. Our method distinguishes certain and uncertain tasks; the baseline lists only certain ones.

\startpara{Procedures}
The study followed a within-subject design, where each participant completed two trials—one using our method, and one using the baseline. Each trial included two questions per query type (“When,” “Why Not,” and “What”), totaling 12 questions. Method order was counterbalanced across participants to mitigate ordering effects. The study included a demonstration, attention-check questions, bonus incentives, and timing to ensure data quality.  

\startpara{Dependent Measures}
User performance was measured by the number of correctly answered prediction questions, reported separately for each query type. Response time per question was also recorded.
After each trial, participants rated explanation quality on a 5-point Likert scale~\cite{hoffman2018metrics} using the same seven metrics described in \sectref{sec:study-sum}.

\startpara{Hypotheses}
We hypothesized that, compared to the baseline, our generated explanations would (\textbf{H3}) improve user question-answering performance and (\textbf{H4}) receive higher ratings on explanation quality metrics.

\startpara{Results}
As shown in \figref{fig:expCorrectRating}, users answered significantly more questions correctly using our HDE explanations compared to the baseline across all three query types. 
Paired t-tests ($\alpha$ = 0.05) confirm the improvement for \textit{when} (t(20)=9.65, p $\leq$ 0.01, d=2.16), 
\textit{why not} (t(20)=13.23, p $\leq$ 0.01, d=2.96), 
and \textit{what} (t(20)=12.05, p $\leq$ 0.01, d=2.69).
\emph{The data supports H3.}

\begin{figure}[t]
    \centering
    \includegraphics[width=1.0\columnwidth]{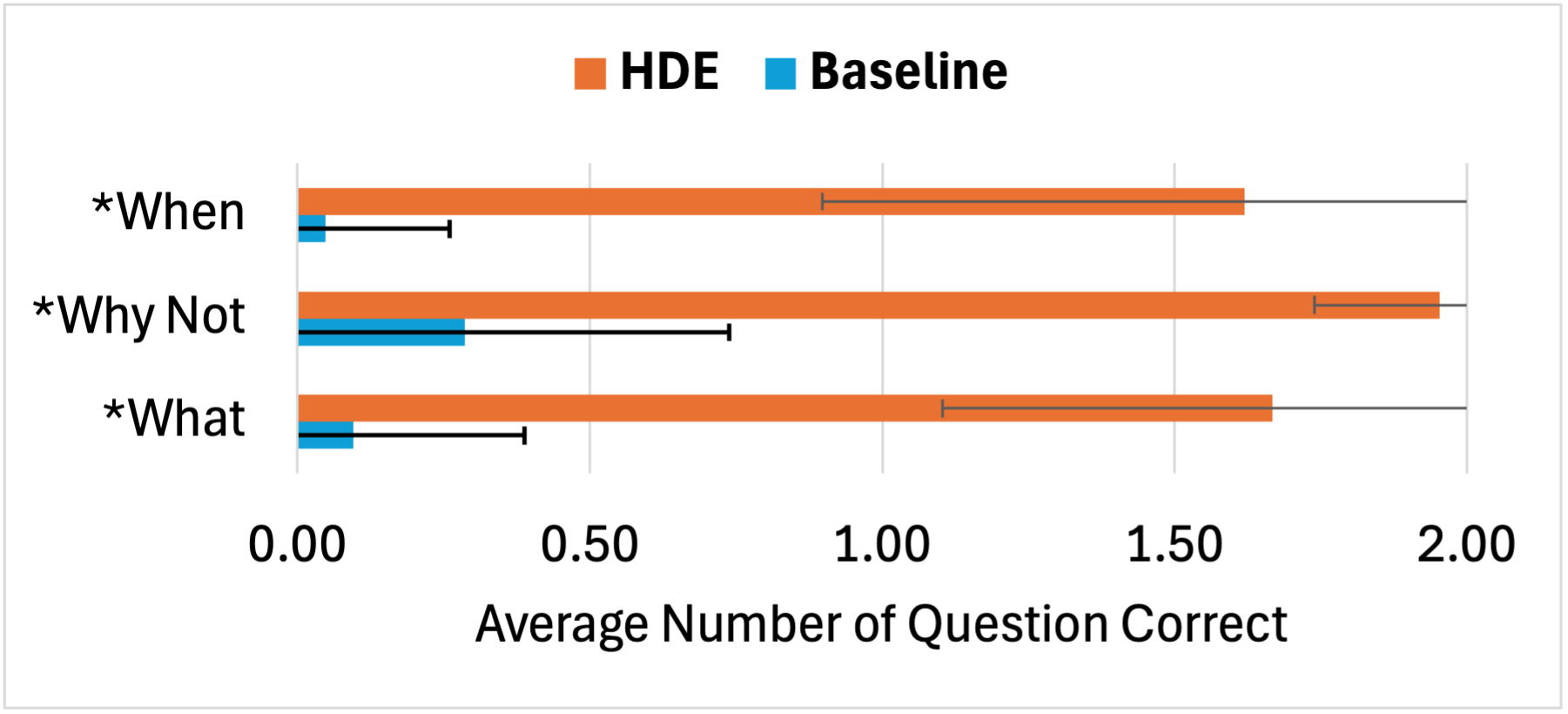}
    \caption{Mean and standard deviation of participant performance on query-based explanations (* indicates statistically significant difference).}
    \label{fig:expCorrectRating}
\end{figure}

\figref{fig:expGoodnessRating} shows participant ratings on explanation quality. 
Wilcoxon signed-rank tests ($\alpha$ = 0.05) indicate significantly higher ratings for HDE across all seven metrics: 
\textit{understanding} (W=3.5, Z=-3.02, p $\leq$ 0.01, r=-0.47), 
\textit{satisfaction} (W=5.0, Z=-3.01, p $\leq$ 0.01, r=-0.46), 
\textit{detail} (W=4.5, Z=-3.02, p $\leq$ 0.01, r=-0.47), 
\textit{completeness} (W=0.0, Z=-3.21, p $\leq$ 0.01, r=-0.49), 
\textit{actionability} (W=4.0, Z=-2.53, p $\leq$ 0.02, r=-0.39), 
\textit{reliability} (W=3.0, Z=-2.78, p $\leq$ 0.01, r=-0.43), 
and \textit{trust} (W=0.0, Z=-2.68, p $\leq$ 0.01, r=0.41). 
\emph{The data supports H4.}

\begin{figure}[t]
    \centering
    \includegraphics[width=1.0\columnwidth]{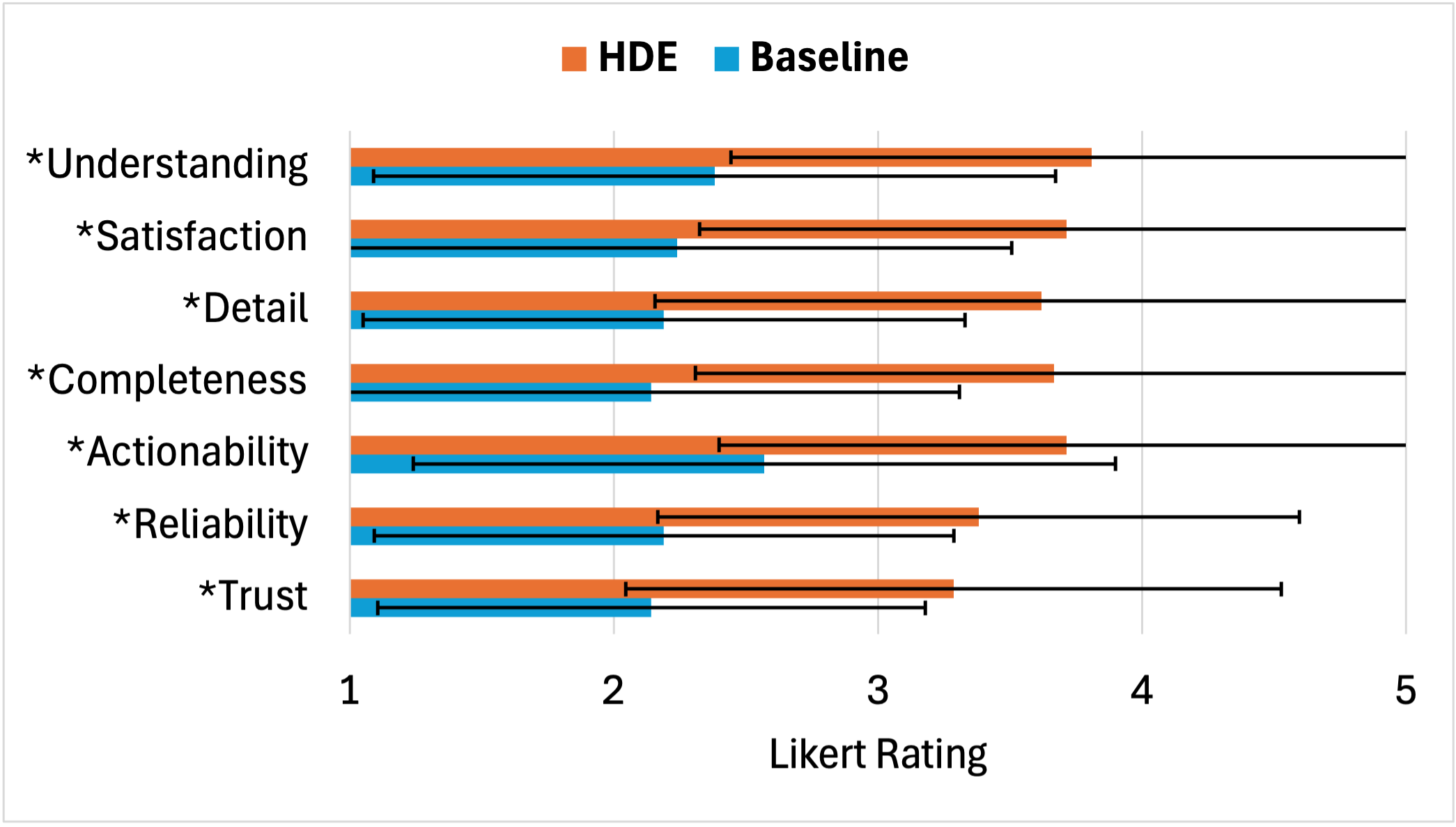}
    \caption{Mean and standard deviation of participant ratings on query-based explanations (* indicates statistically significant difference).}
    \label{fig:expGoodnessRating}
\end{figure}

\startpara{Discussion}
The improvement in user performance is likely due to our explanation method’s ability to present both certain and uncertain features, enabling users to more accurately predict if and what tasks will occur. In contrast, the baseline captures only agent-specific behavior and lacks inter-agent context. It may omit shared task dependencies or include irrelevant conditions observed by individual agents, requiring users to reconcile fragmented information—reducing its effectiveness in decentralized settings.

Participants also gave higher subjective ratings to our explanations across all goodness metrics, suggesting they valued access to decentralized task dependencies and uncertainty, even if the explanations were more complex. Notably, the inclusion of uncertain features did not increase response time, indicating that users could efficiently interpret the richer information.

\section{Conclusion} \label{sec:conclu} 

We presented novel approaches for summarizing and explaining decentralized MARL policies. 
Computational experiments across four MARL domains and two learning algorithms show that our method is scalable and efficient, generating compact summarizations and meaningful explanations even in large environments with many agents and tasks. 
User studies further demonstrate that our approach improves user performance and perceived explanation quality, without increasing response time.

Future work includes integrating these explanations into interactive human-agent systems, supporting more expressive user queries, and leveraging large language models to enhance explanation clarity and usability.

\clearpage

\section*{Acknowledgements}
This work was supported in part by the U.S. National Science Foundation grant CCF-1942836, and Israel Ministry of Innovation, Science \& Technology grant 1001818511. The opinions, findings, conclusions, or recommendations expressed in this material are those of the author(s) and do not necessarily reflect the views of the sponsoring agencies.

\bibliography{references}

\appendix

\section{Proof of \thmref{thm:hds}} \label{app:proof}

\setcounter{theorem}{0}
\begin{theorem}
Given a set of agent trajectories $\{\omega^i\}_{i=1}^N$ produced by executing decentralized MARL policies $\{\pi^i\}_{i=1}^N$ in a single episode, the Hasse diagram $\cD = (\cV, \cE)$ constructed by \agref{ag:hds} is both a correct and complete policy summarization. 
\end{theorem}

\begin{proof}
\textbf{(Correctness)} Suppose, for contradiction, that the Hasse diagram $\cD = (\cV, \cE)$ produced by \agref{ag:hds} is not a correct policy summarization. Then there exists a path $\rho$ through $\cD$ such that its projection $\rho^i$ onto some agent $i$ does not conform to the agent’s task sequence $\trace{\omega^i}$.

This implies that in $\rho^i$ there is a task $\tau^i_k$ that precedes another task $\tau^i_{k'}$ with $k > k'$, yet $\tau^i_k$ appears earlier in the path than $\tau^i_{k'}$. Let $v$ and $v'$ be the nodes in $\rho$ that contain $\tau^i_k$ and $\tau^i_{k'}$, respectively. Since $\rho$ is a valid path, it must be that $v \prec v'$, meaning tasks in $v$ occur before those in $v'$. But then $\tau^i_k$ should occur before $\tau^i_{k'}$ in $\trace{\omega^i}$, contradicting $k > k'$. Hence, all projections must conform, and the diagram is correct.

\textbf{(Completeness)} The algorithm explicitly iterates over each trajectory $\omega^i$ and processes its task sequence $\trace{\omega^i}$ in order. For every consecutive pair of tasks $(\tau^i_{k-1}, \tau^i_k)$, an edge is added from the vertex containing $\tau^i_{k-1}$ to the one containing $\tau^i_k$. Thus, for each agent, a path through the DAG is constructed that corresponds exactly to $\trace{\omega^i}$.

In the transitive reduction step, an edge $(v, v')$ is removed only if there already exists an alternative path from $v$ to $v'$. Hence, the path corresponding to $\trace{\omega^i}$ is preserved. Therefore, for every agent $i$, there exists at least one path $\rho$ in $\cD$ such that $\rho^i = \trace{\omega^i}$, proving completeness.
\end{proof}

\section{Algorithm for “Why Not” Queries} \label{app:why}

The pseudocode below outlines our method for answering “Why don’t agents $\mathcal{G}_q$ do task $\tau_q$ under conditions $\Phi_q$?” queries, as described in \sectref{sec:why}. 

\begin{algorithm}[hb]
\caption{``Why Not'' Query-Based Explanation}
\label{ag:why}
{\small
\textbf{Input:} Agent group $\mathcal{G}_q$, query task $\tau_q$, query conditions $\Phi_q$, Hasse diagrams $\{\mathcal{D}_j\}$, and feature set $\mathcal{F}$ \\
\textbf{Output:} Language-based explanation $\cX$
\begin{algorithmic}[1]
\STATE Extract relevant features $\mathcal{F}_q \subseteq \mathcal{F}$ for $\tau_q$
\STATE Initialize uncertainty dictionary $U \leftarrow \emptyset$
\FOR{each diagram $\mathcal{D}_j$}
  \IF{$\tau_q$ completed by $\mathcal{G}_q$ in $\mathcal{D}_j$}
    \STATE Let $v_\tau$ be the node where $\tau_q$ is completed
    \STATE Compute partial comparability graph from $v_\tau$
    \FOR{each node $v$ not reachable to/from $v_\tau$}
      \STATE Add features associated with $v$ to $U[\mathcal{D}_j]$
    \ENDFOR
  \ENDIF
\ENDFOR
\STATE Encode $\Phi_q$ as boolean vector $B_1$ over $\mathcal{F}_q$
\STATE Encode non-target nodes as boolean vectors $B_0$ using $U$
\STATE Apply Quine-McCluskey to distinguish $B_1$ from $B_0$
\STATE Translate resulting formula into explanation $\cX$
\RETURN $\cX$
\end{algorithmic}
}
\end{algorithm}

\section{Details on Computational Experiments} \label{app:exp}

\startpara{Baseline Summarization Visualization}
\figref{fig:baseline-sum} shows an example of the baseline summarization used throughout our evaluations. Each agent’s abstract policy graph summarizes its possible task sequences over 100 execution episodes, with sequence-level probabilities. These per-agent graphs are displayed side-by-side, requiring users to manually infer inter-agent task coordination, which is a key limitation compared to our HDS method.

\begin{figure}[ht]
    \centering
    \includegraphics[width=\columnwidth]{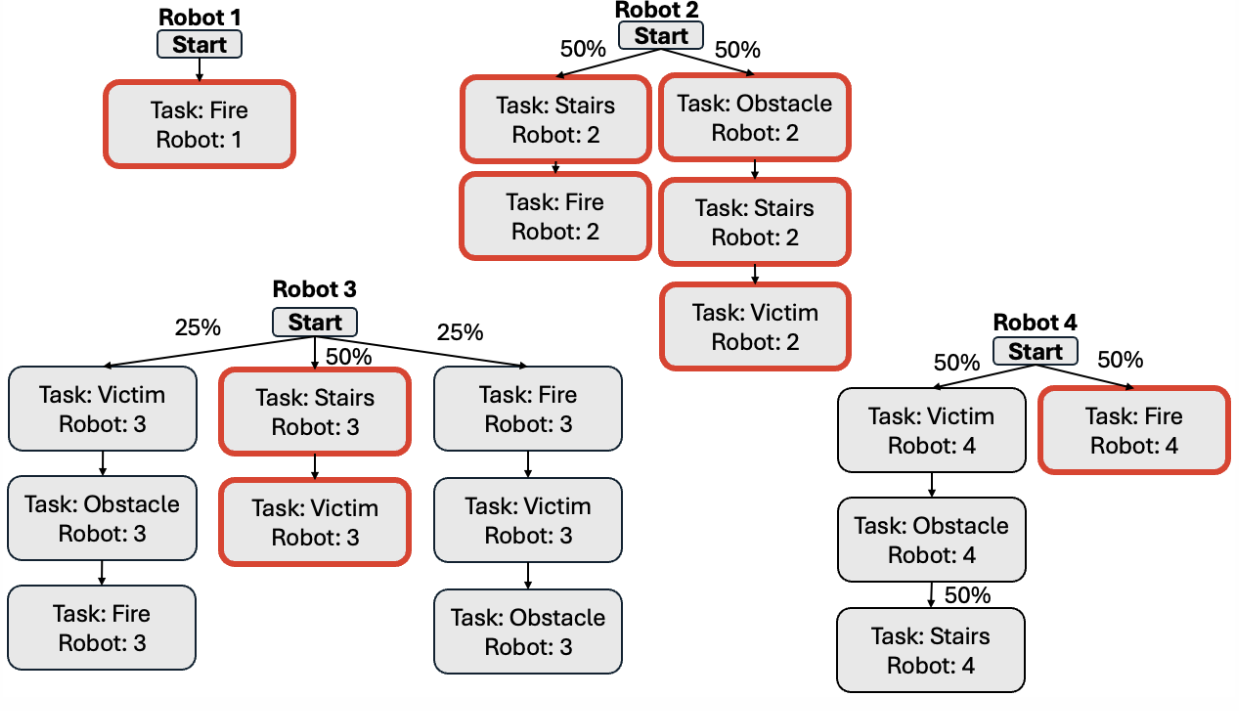}
    \caption{Example of baseline summarization. }
    \label{fig:baseline-sum}
\end{figure}

\startpara{DTDE Policy Summarization Results}
\tabref{tab:sum-dtde} reports summarization statistics for HDS and the baseline using DTDE policies trained with IA2C. Due to convergence issues of training IA2C in larger environments, we report results only for LBF(5,5) and RW(3,4), rather than LBF(9,9) and RW(4,19) used in the CTDE setting with SEAC. As in all CTDE cases, HDS yields substantially more compact summarizations, whereas the baseline produces significantly larger graphs.

\begin{table}[h]
\centering
\begin{tabular}{l|cc|cc}
\toprule
\multicolumn{1}{c}{Domain} & \multicolumn{2}{c}{\textbf{HDS}} & \multicolumn{2}{c}{\textbf{Baseline}} \\
\cmidrule(lr){2-3} \cmidrule(lr){4-5}
\multicolumn{1}{c}{$(N, |T|)$} & $|\cV|$ & $|\cE|$ & $|\cV|$& $|\cE|$ \\
\midrule
SR (9,7)  & 8 & 7.79 & 350 & 341 \\
LBF (5,5) & 6 & 5.62 & 151 & 146 \\
RW (3,4) & 5 & 4 & 69 & 66 \\
PP (7,6)  & 7 & 6 & 107 & 100 \\
\bottomrule
\end{tabular}
\caption{Summarization sizes for HDS and the baseline on selected domain settings, based on 100 episodes executed using DTDE policies trained with IA2C. }
\label{tab:sum-dtde}
\end{table}

\startpara{Hasse Diagram Diversity}
\figref{fig:piechart} shows the distribution of Hasse diagram types across 100 episodes for selected domain-policy pairs. Although individual executions often yield unique diagrams, they typically fall into a small number of structural categories based on edge count. For instance, SR(9,7) has 100 unique diagrams but only 6 distinct edge counts. This highlights HDS's ability to preserve behavioral variation while maintaining structural compactness.

\begin{figure}[ht]
    \centering
    \includegraphics[width=0.8\columnwidth]{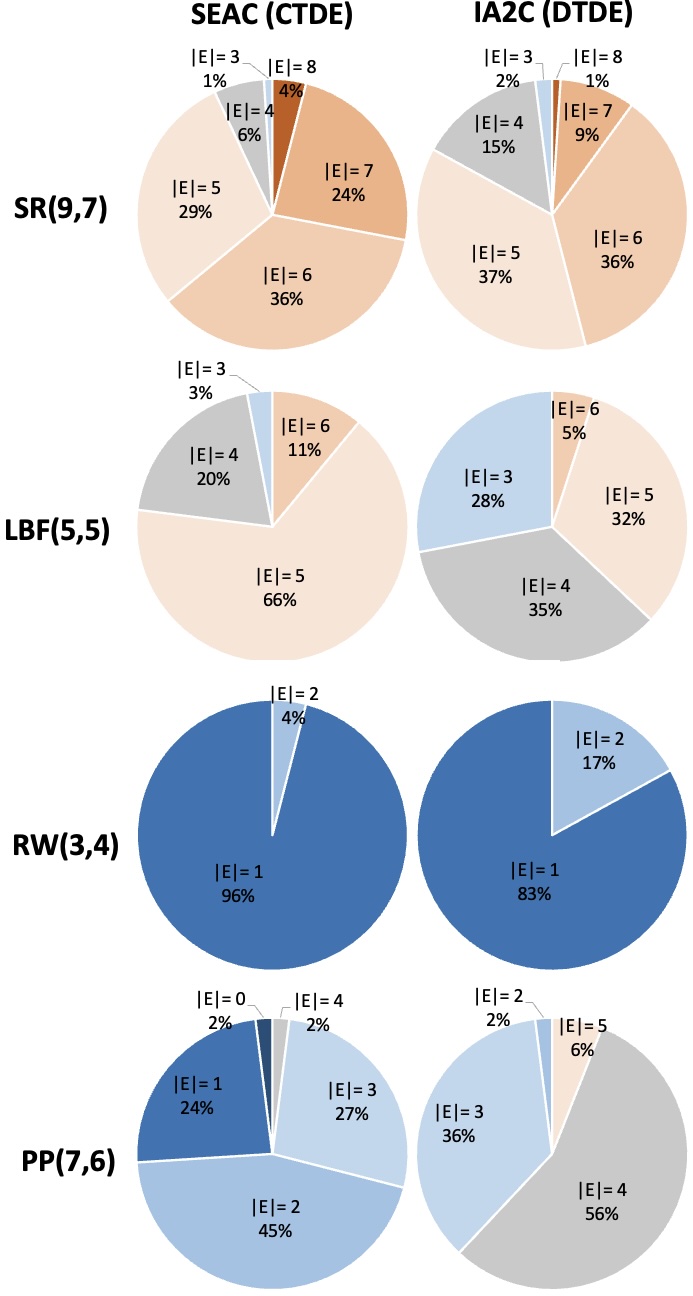}
    \caption{Distribution of Hasse diagram types across 100 episodes for each domain and training method. Each slice represents a unique diagram type categorized by edge count, illustrating structural diversity captured by HDS.}
    \label{fig:piechart}
\end{figure}

\startpara{Additional Explanation Results}
\tabtabtabref{tab:when-dtde}{tab:what-dtde}{tab:what-ctde} provide additional explanation statistics for HD-When and HD-What under DTDE policies, and HD-What under CTDE policies. These tables report the number of certain and uncertain features or tasks included in the explanations, based on 100 episodes. Results in \tabref{tab:when-dtde} show trends consistent with those observed in \tabref{tab:when-ctde} for CTDE “When” queries: our methods produce compact, structured explanations with meaningful uncertainty, while the baseline includes more features but no uncertainty terms. 
Results in \tabtabref{tab:what-dtde}{tab:what-ctde} show explanations with increased information due to captured coordination and task order not present in the baseline.

\begin{table}[ht]
\centering
\begin{tabular}{l|cc|cc}
\toprule
\multicolumn{1}{c}{Domain} & \multicolumn{2}{c}{\textbf{HD-When}} & \multicolumn{2}{c}{\textbf{Baseline}} \\
\cmidrule(lr){2-3} \cmidrule(lr){4-5}
\multicolumn{1}{c}{$(N, |T|)$} & $|\cF^c|$ & $|\cF^u|$ & $|\cF^c|$& $|\cF^u|$ \\
\midrule
SR (9,7)  & 5 & 5 & 45 & 0 \\
LBF (5,5) & 14 & 4 & 41 & 0 \\
RW (3,4) & 0 & 2 & 2 & 0 \\
PP (7,6)  & 8 & 3 & 19 & 0 \\
\bottomrule
\end{tabular}
\caption{Explanations sizes for our method and the baseline on selected domain settings, based on 100 episodes executed using DTDE policies trained with IA2C. }
\label{tab:when-dtde}
\end{table}
\begin{table}[ht]
\centering
\begin{tabular}{l|cc|cc}
\toprule
\multicolumn{1}{c}{Domain} & \multicolumn{2}{c}{\textbf{HD-What}} & \multicolumn{2}{c}{\textbf{Baseline}} \\
\cmidrule(lr){2-3} \cmidrule(lr){4-5}
\multicolumn{1}{c}{$(N, |T|)$} & $|\cT^c|$ & $|\cT^u|$ & $|\cT^c|$& $|\cT^u|$ \\
\midrule
SR (9,7)  & 1 & 7 & 3 & 0 \\
LBF (5,5) & 4 & 0 & 3 & 0 \\
RW (3,4) & 1 & 3 & 2 & 0 \\
PP (7,6)  & 2 & 3 & 2 & 0 \\
\bottomrule
\end{tabular}
\caption{Explanations sizes for our method and the baseline on selected domain settings, based on 100 episodes executed using DTDE policies trained with IA2C. }
\label{tab:what-dtde}
\end{table}
\begin{table}[ht]
\centering
\begin{tabular}{l|cc|cc}
\toprule
\multicolumn{1}{c}{Domain} & \multicolumn{2}{c}{\textbf{HD-What}} & \multicolumn{2}{c}{\textbf{Baseline}} \\
\cmidrule(lr){2-3} \cmidrule(lr){4-5}
\multicolumn{1}{c}{$(N, |T|)$} & $|\cT^c|$ & $|\cT^u|$ & $|\cT^c|$& $|\cT^u|$ \\
\midrule
SR (9,7)  & 3 & 0 & 2 & 0 \\
LBF (9,9) & 8 & 0 & 6 & 0 \\
RW (4,19) & 9 & 11 & 10 & 0 \\
PP (7,6)  & 1 & 5 & 2 & 0 \\
\bottomrule
\end{tabular}
\caption{Explanations sizes for our method and the baseline on the largest setting in each domain, based on 100 episodes executed using CTDE policies trained with SEAC.}
\label{tab:what-ctde}
\end{table}

\section{Details on User Studies} \label{app:study}

\startpara{Summarization Study Interface Example}
\figref{fig:hds} shows an example of the user interface used in the summarization study for presenting HDS-generated summaries. Each summary consists of a table showing agent-task assignments and a list of natural language rules representing partial task order, converted from a Hasse diagram. The interface displays the top three most frequent summaries (based on 100 episodes), with empirical likelihoods shown beneath each. The most likely summary is visually emphasized with a red border.

\begin{figure}[ht]
    \centering
    \includegraphics[width=\columnwidth]{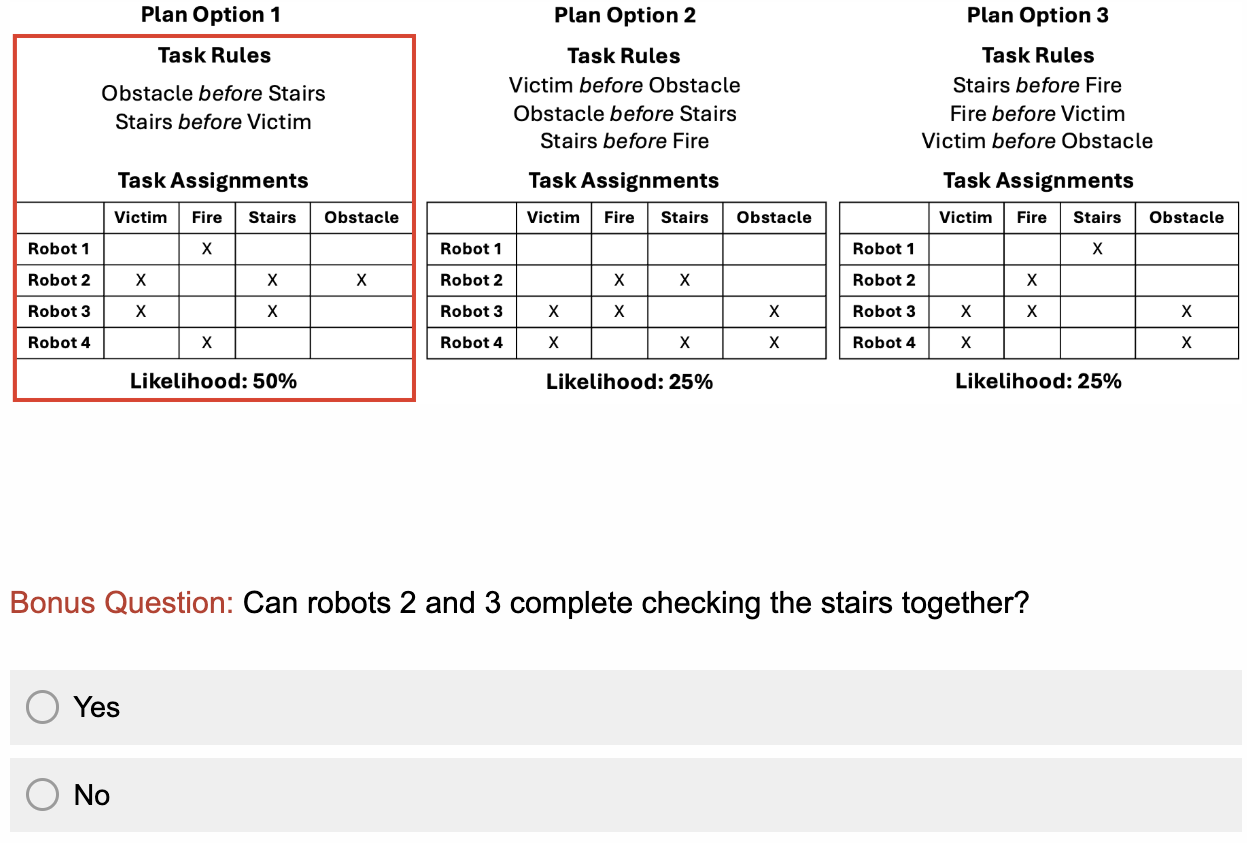}
    \caption{Example user interface displaying an HDS-generated summarization and associated question.}
    \label{fig:hds}
\end{figure}

\startpara{Explanation Study Interface Examples}
\figfigfigref{fig:hd-when}{fig:hd-why}{fig:hd-what} show example user interfaces used in the explanation study for the “When,” “Why Not,” and “What” queries, respectively. Each interface presents a scenario with agents and tasks, along with an HD-generated explanation and a follow-up question for participants to answer.

\begin{figure}[ht]
    \centering
    \includegraphics[width=0.9\columnwidth]{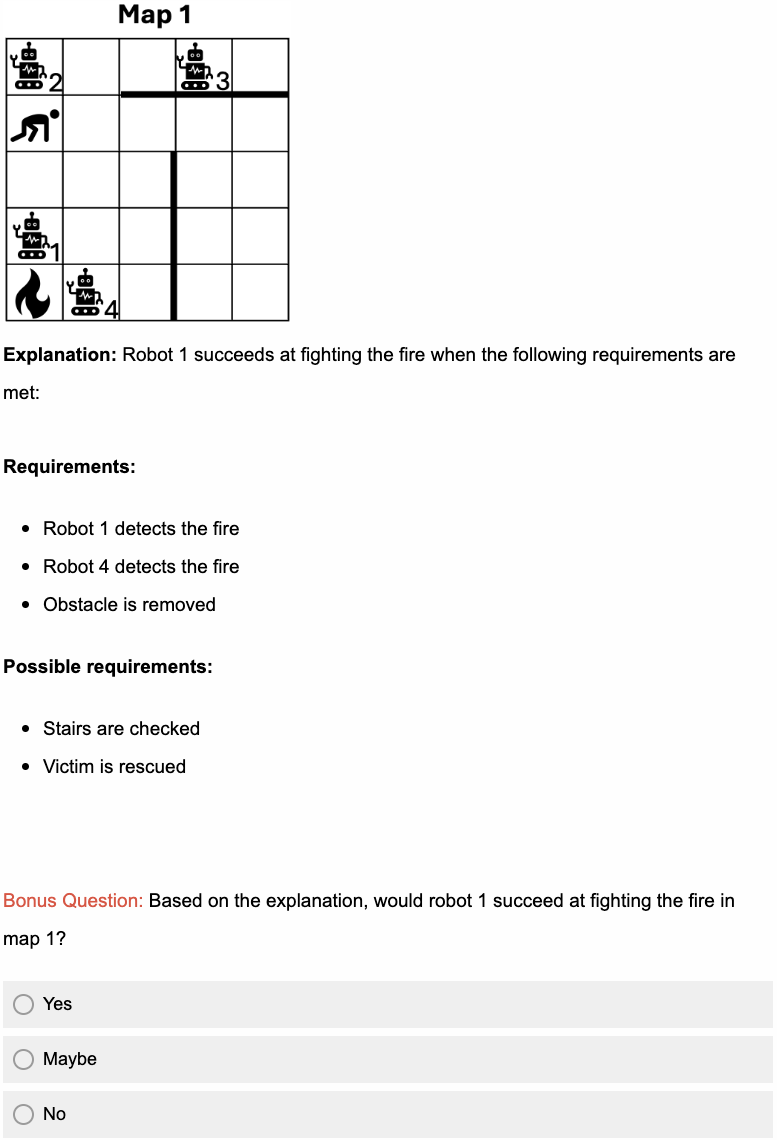}
    \caption{Example user interface displaying an HD-When explanation and associated question.}
    \label{fig:hd-when}
\end{figure}

\begin{figure}[ht]
    \centering
    \includegraphics[width=0.9\columnwidth]{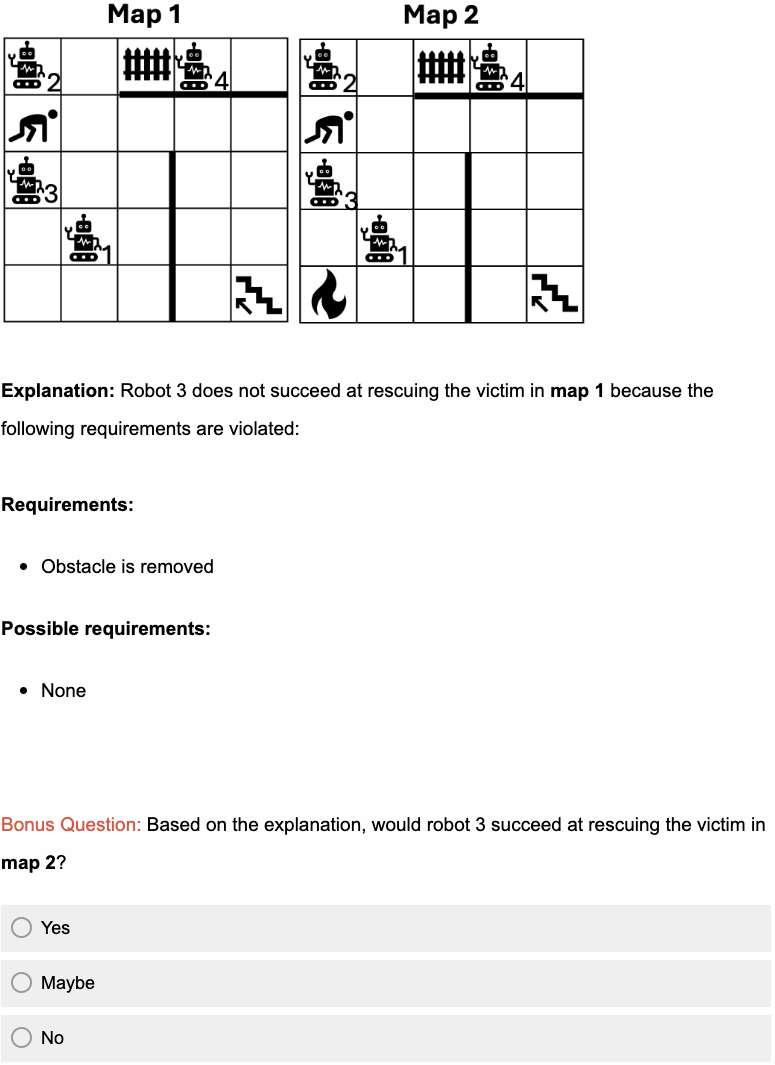}
    \caption{Example user interface displaying an HD-WhyNot explanation and associated question.}
    \label{fig:hd-why}
\end{figure}

\begin{figure}[ht]
    \centering
    \includegraphics[width=0.9\columnwidth]{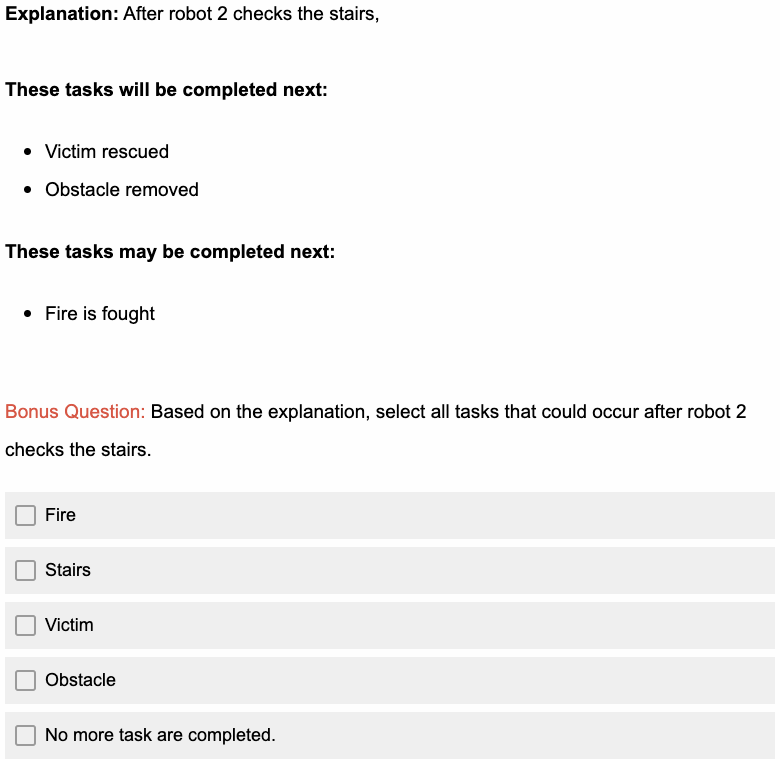}
    \caption{Example user interface displaying an HD-What explanation and associated question.}
    \label{fig:hd-what}
\end{figure}

\end{document}